\documentclass{new_tlp}
\usepackage{times}  
\usepackage{helvet}  
\usepackage{courier}  
\usepackage{url}  
\usepackage{graphicx}  
\usepackage{amsmath}
\usepackage{amssymb}
\usepackage{amsfonts}
\usepackage{listings}

\newcommand{\as}{``}

\newcommand{\be}{\begin{em}}
	\newcommand{\ee}{\end{em}}
\newcommand{\bb}{\begin{bf}}
	\newcommand{\eb}{\end{bf}}
\newcommand{\I}[1]{\relax\ifmmode\mbox{\it#1}\else{\it#1}\fi}

\newcommand{\tbs}{\hspace*{4mm}}
\newcommand{\tbm}{\hspace*{8mm}}
\newcommand{\tbl}{\hspace*{12mm}}
\newcommand{\no}{not\,}

\newcommand{\rif}{~\ref}

\newcommand{\K}{\mbox{\textbf{K}}}
\newcommand{\M}{\mbox{\textbf{M}}}
\newcommand{\N}{\mbox{\textbf{\no}}\,}
\newcommand{\NO}{\mbox{\textbf{NOT}}\,}
\newcommand{\Kw}{\mbox{\textbf{K}$^W$}}
\newcommand{\Mw}{\mbox{\textbf{M}$^W$}}
\newcommand{\Nw}{\mbox{\textbf{\no}$^W$}}
\newcommand{\NOw}{\mbox{\textbf{NOT}$^W$}\,}
\newcommand{\A}{{\cal A}}
\newcommand{\Sc}{{\cal S}}
\def\naf{not\,}
\def\ar{\leftarrow}
\newcommand{\tabp}{\mbox{${\cal{T}}(\Pi)$}}

\newtheorem{definition}{Definition}[section]

\newtheorem{corollary}{Corollary}[section]

\newtheorem{theorem}{Theorem}[section]

\title{About epistemic negation and world views\\
	in Epistemic Logic Programs}

\author[S. Costantini]
{Stefania Costantini\\
	DISIM - Universit{\`a} dell'Aquila, Italy\\
	\email{stefania.costantini@univaq.it}}

\begin{document}
\maketitle

\begin{abstract}
In this paper we consider Epistemic Logic Programs, which extend Answer Set Programming (ASP) with \as epistemic operators'', and a recent approach to the semantics of such programs in terms of \emph{World Views}. We propose some observations on the existence and number of world views. We exploit an extended ASP semantics in order to: (i) provide a novel characterization of world views; (ii) query world views and query the whole set of world views. This paper is under consideration for acceptance in TPLP
\end{abstract}


\section{Introduction}

Answer Set Programming (ASP) is a successful logic programming paradigm under the answer set semantics (AS) \cite{GelLif88,GelLif91}. ASP has been applied, e.g., to information integration, constraint satisfaction, routing, planning, diagnosis, configuration, computer-aided verification, biology/biomedicine, knowledge management, etc. (cf. \cite{ASPJournal2016} and the references therein).
The ASP approach to problem-solving consists basically in the following: (i) encoding of the given problem via an ASP program; (ii) computing the \as answer sets'' of such a program via an inference engine, or \as ASP solver''; (iii) extracting the problem solutions by examining such answer sets; in fact, answer set contents can be in general reformulated in order to present the solution in terms of the given problem. 

Epistemic Logic programs (ELPs), first introduced in \cite{Gelfond91}, extend ASP with \emph{epistemic operators} \K\ and \M: $\K A$ means that (ground) atom $A$ is true in every answer set of given program  $\Pi$, whereas $\M A$ means that $A$ is true in some of the answer sets of $\Pi$. Thus, such operators introduce a form of reasoning over multiple answer sets, where the collection of all the answer sets of $\Pi$ is called a \emph{world view} and, if \M\ and \K\ occur in $\Pi$, more than one world view may exist. The  \emph{epistemic negation operator} \N\  expresses that $A$ \emph{is not provably true}, meaning that $A$ is false in at least one answer set of $\Pi$. By means of \N one can define both \K\ and \M: in fact, $\K A$ and $\M A$ can be rephrased as $\no \N A$ and $\N \no A$, respectively, where $\no$ is ASP standard `default negation'. In fact: concerning \K, if it is not true that $A$ is false in some answer set then $A$ must be true in all of them; concerning \M, if $\no A$ is false in some answer set then $A$ must be true in that answer set. Consequently, many approaches to ELPs, e.g., \cite{ShenE16}, consider explicitly only the operator \N\hspace{-0.1cm}. The semantics of ELPs is provided in terms of \emph{World Views}: instead of a unique set of answer sets like in ASP, there is now a set of such sets. Each one, called \as world view'', consistently satisfies the epistemic expressions that appear in a given program. An epistemic logic program may have several world views. ELP solving systems have been defined and implemented \cite{SonLKL17} on top of state-of-the-art ASP solvers, that are invoked (more than once) to generate and check potential world views. 

In this paper we refer to the work of \cite{ShenE16} concerning the \as general epistemic semantics'' of ELPs. Many other semantic approaches/characterizations have been introduced, among which the following: \cite{Gelfond11,Truszczynski11}, \cite{CerroHS15,Su17}, \cite{Kahl18,Su19,Cabalar19a,Cabalar19b}; the aim is essentially to avoid unintended world views, but also to propose extensions of the basic paradigm. We consider the semantics introduced in \cite{ShenE16} because: it takes properly into account previous relevant work, and copes in a satisfactory way with most of the controversial examples; moreover, a \as plus'' of this approach is that it provides a very general characterization of world views, which is applicable to every variant of the AS semantics. Concerning practical applicability, the work of \cite{Woltran2018} introduces an effective method for the practical characterization and computation of world views under the semantics of \cite{ShenE16}. 

The traditional example proposed in the literature to illustrate and motivate the use of \N is a re-elaboration of a famous example by Bowen and Kowalski \cite{BowKow1983}:

\smallskip
$
\begin{array}{l}
\mathit{innocent(X)} \ar \mathit{suspect(X)}, \no \mathit{demo(guilty(X))}.
\end{array}
$

stating that one is innocent if (s)he cannot be \emph{proved} to be guilty (clearly, $X$ will be instantiated to some constant defined elsewhere in the program, say $\mathit{john}$). This example was formulated to support the introduction of prolog metainterpreters to implement the $\mathit{demo}$ predicate. However, it can be reasonable to assume the adoption of ASP to encode available knowledge about the underlying investigative case: in fact, this knowledge (collected by investigators) can often be incomplete, uncertain, etc. If using ASP however, a formulation such as 

\smallskip
$
\begin{array}{l}
\mathit{innocent(X)} \ar \mathit{suspect(X)}, \no \mathit{guilty(X)}.
\end{array}
$

\noindent 
does not suffice, as there can be answer sets where the suspect is deemed innocent and others where (s)he is deemed guilty, thus not allowing a unique conclusion to be drawn. By exploiting epistemic negation the formulation becomes the following: 

\smallskip
$
\begin{array}{l}
\mathit{innocent(X)} \ar \mathit{suspect(X)}, \N \mathit{guilty(X)}.
\end{array}
$

In particular, it states that one is innocent if not provably guilty: i.e., (s)he innocent if there exists some answer set where (s)he is not guilty. In fact, the existence of such an answer set introduces a \emph{reasonable doubt}. In Section\rif{casestudy} we provide and discuss a simple though plausible underlying description of the case at hand where, due to uncertain knowledge (that in practical investigations occurs very often) the resulting ELP program has more than one world view.

\smallskip\noindent{\bf Our Contributions:}
\begin{enumerate}
	\item We propose some observations about programs with epistemic negation, and in particular on the existence and number of world views.
	\item We introduce \emph{epistemic scenarios} that are easy to obtain, and among which are the \emph{valid guesses}, i.e., the hypotheses about the truth value of epistemic literals that determine world views. We establish, for the first time in the literature, an upper bound on the number of valid guesses (and thus of world views). This is of theoretical interest, as no estimation of the number of world views had ever been provided before. Also, the notion of epistemic scenarios reduces (in many cases very significantly) the number of guesses to be checked for validity in order to be able to compute world views. 
	\item We define a new method to check validity of a guess, which is particularly convenient when the given program contains few epistemic literals.
	\item We provide for the first time in the literature a device for top-down (prolog-style) query-answering concerning either a single world view, or also the whole set of world views, which does not require to compute the world views beforehand.
\end{enumerate}

Querying world views may be useful in reference to the above example, as an investigator/lawyer/judge may wish to pose a query about $\mathit{john}$ without computing a whole world view. 
Moreover, in case several world views exist, the involved parties may wish to query whether $\mathit{john}$ is innocent/guilty in some/every world view. We in fact provide several operators to query world views of given ELP program. $\Nw A$ means that $\N A$ holds in some world view of $\Pi$, $\Kw A$ means that $A$ is true in every world view (where, according to  \cite{ShenE16}, $A$ is true in a world view if $A$ belongs to all the answer sets that compose the world view). $\Mw A$ means that $A$ is possible in at least one world view (i.e., $A$ belongs to at least one of the composing answer sets). An enhanced \Mw can check whether $A$ is possible in \emph{every} world view. $\NO A$ is a shorthand for $\neg\M A$, meaning that $A$ is false in every answer set of a world view and $\NOw A$ means falsity of $A$ in every answer set of every world view. The usefulness of these operators will be illustrated in Section\rif{casestudy} where we further elaborate on the above example.

In general terms, the possibility of querying world views may have significant practical applications, for instance in Intelligent Agents\footnote{cf. \cite{BordiniBDFGLOPR06} for a survey about logical approaches to the definition of intelligent agents.}, recommender systems, decision support systems: such systems might employ answer set modules for reasoning tasks \cite{Cos2012ASP-modules}, and might prefer to receive answers to queries rather that having to inspect all the answer sets/world views. {\bf Notice that:} \emph{query-answering might be simulated by suitable APIs defined on top of a full world-view computation. However, this would not be resilient to modifications of a program's knowledge base, as each modification would require to recompute the world views. Our query-answering facility does not need to compute full world views, and new queries automatically consider the updated knowledge. This is important in evolving systems such as, e.g., agents.}

We resort to Resource-based Answer Set Semantics (RAS, cf. \cite{CostantiniF15}) for formalizing the proposed new characterization of world views.
RAS provides answer sets to all programs, including those which are inconsistent under AS. For consistent programs, which are those for which epistemic negation makes sense as some world view may exist, RAS returns the same answer sets as AS (under very simple conditions, seen below) and therefore the same world views. RAS however features, differently from AS, a prolog-like top-down query answering facility \cite{CostantiniF16}. 

The paper is organized as follows. In Sections\rif{asp}--\ref{elp} we recall Answer Set Programming, and Epistemic Logic Programs under the semantics of \cite{ShenE16} (we assume a basic knowledge on logic programming and its declarative and procedural semantics as illustrated in standard textbooks, e.g., \cite{lloyd87}). Then, in Section\rif{observations} we propose some observations on ELPs that will lead to an alternative approach to the computation of world view via \emph{epistemic scenarios}. We introduce Resource-Based Answer set semantics in Section\rif{ras}, and then we show in Section\rif{elpras} how by means of its associated query-answering device, RAS-XSB-resolution \cite{CostantiniF16}, we are able to query ELPs so as to check the validity of guesses, and to query world views and sets of world views. In Section\rif{casestudy} we expand the case study introduced above, and finally we conclude. In this paper we discuss semantics and applications, but we treat implementation issues only in general terms. We intend in fact to provide here novel contributions that can be of theoretical interest by themselves, and can also in perspective be useful to developers and users of ELP programs and related applications. Coping with efficiency issues and performing experiments is deferred to future work.

\section{Answer Set semantics (AS) and Answer Set Programming (ASP)}
\label{asp}

\as Answer Set Programming'' (ASP) is a successful programming paradigm based on the Answer Set Semantics
\cite{GelLif88,GelLif91}. In ASP, one can see an answer set program (for short, just \as program'') as a set of
statements that specify a problem, where each answer set
represents a solution compatible with this specification. Whenever a program has no answer sets (no solution can be found), it is said to be \emph{inconsistent}, otherwise it is said to be \emph{consistent}. 
Several well-developed freely available \emph{answer set solvers} exist~\cite{ASPJournal2016},
that compute the answer sets of a given program. 
Syntactically, an ASP program $\Pi$ is a
collection of \emph{rules} of the form
`\(H\leftarrow\; A_{1} , \ldots , A_m , \naf\,A_{m+1}, \ldots, \naf\,A_{m+n}.\)'
where $H$ is an atom, $m,n\geqslant 0$, and each $A_i$, $i \leq m+n$, is
an atom. Atoms and their negations are called \emph{literals}.
The left-hand side and the right-hand side of the clause are called
\emph{head} and \emph{body}, respectively. A rule with empty body is called a \emph{fact}. 
A rule with empty head is
a \emph{constraint}, where a constraint of the form
`\(\leftarrow L_1,...,L_n.\)'
states that literals $L_1,\ldots,L_n$ cannot be simultaneously true
in any answer set. Constraints are often rephrased as `\(f\leftarrow \naf f, L_1,...,L_n.\)'
where $f$ is a fresh atom. To avoid the contradiction over $f$, some of the $L_i$'s must
be false thus forcing $f$ to be false, and this, if achieved, fulfills the constraint.

There are other features that for the sake of simplicity we do not consider in this paper,
namely: (i) explicit disjunction; (ii) so-called classical negation $\neg A$, that can be however easily compiled away \cite{GelLif91}; (iii) advanced programming features such as aggregates; (iv) function symbols. In the rest of the paper, as it is customary in the ASP literature and as is done in \cite{ShenE16}, we will implicitly refer to the
\as ground'' version of $\Pi$, which is obtained by replacing in all possible
ways the variables occurring in $\Pi$ with the constants occurring in $\Pi$ itself,
and is thus composed of ground atoms, i.e., atoms which contain no variables. 

The answer set (or \as stable model'') semantics (AS) can be defined in several ways. However, answer sets of a program $\Pi$, if any exists, are supported minimal classical models of the program interpreted as a first-order theory in the obvious way. The original definition by \cite{GelLif88} was in terms of the `GL-Operator' $\Gamma$ where, given set of atoms $I$ and program $\Pi$,
$\Gamma_{\Pi}(I)$ is the least Herbrand model of $\Pi^I$, where $\Pi^I$, called the (Gelfond-Lifschitz) `reduct' of $\Pi$ w.r.t. $I$, is a positive program (and so its least Herbrand model can be computed via the immediate consequence operator\footnote{As already mentioned, for general terminology about logic programming a reader may refer to \cite{lloyd87}} $T_P$ applicable to any positive program $P$) obtained from $\Pi$ by:
1. removing from $\Pi$ all rules which contain a negative literal\, $\no{}A$\, such that $A \in  I$; ~ and
2. removing all negative literals from the remaining rules. $I$ is an answer set whenever $\Gamma_{\Pi}(I) = I$.

\section{Epistemic Negation: Semantics}
\label{elp}

As discussed in the Introduction, we can understand Epistemic Logic Programs as answer set programs augmented with the epistemic negation operator \N\hspace{-0.1cm}. In this paper, for the semantics of Epistemic Logic Programs we especially refer to the approach of \cite{ShenE16}, based on the notion of \emph{World Views}: for a given program, instead of a set of answer sets like in ASP, there is now a set of such sets. Each one, called \as world view'', consistently satisfies the epistemic negations occurring in the given program (as well as the other modal expressions possibly defined in terms of \N\hspace{-0.1cm}). An epistemic logic program may admit none or several world views.
World views in \cite{ShenE16} are obtained as follows. Let $\Pi$ be a (ground) epistemic program and let $EP(\Pi)$ be the set of literals of the form $\N F$ occurring in $\Pi$. Given $\Phi \subseteq EP(\Pi)$, the \emph{Epistemic reduct} $\Pi^{\Phi}$ of $\Pi$ w.r.t. $\Phi$ is obtained by: (i) replacing every $\N F \in \Phi$ with true, and (ii) replacing every $\N F \not\in \Phi$ with $\no F$. Then, the set ${\cal A}$ of the answer sets of $\Pi^{\Phi}$ is a \emph{candidate world view} if every $\N F \in \Phi$ is true w.r.t. ${\cal A}$ (i.e., $F$ is false in some answer set $J \in {\cal A}$) and every $\N F \not\in \Phi$ is false (i.e., $F$ is true in every answer set $J \in {\cal A}$). We say that $\A$ is obtained from $\Phi$, or is corresponding to $\Phi$, or that it is a candidate world view w.r.t. $\Phi$, where $\Phi$ is called a \emph{candidate valid guess}. ${\cal A}$ is a \emph{world view} if it is maximal, i.e., there exists no other candidate world view obtained from guess $\Phi'$ where $\Phi \subset \Phi'$, and we call $\Phi$ a \emph{valid guess}. This maximality condition, as discussed in \cite{ShenE16}, is essential for avoiding unintended world views. A literal $F$ is said to be true in $\Pi$ under the general epistemic semantics if $\Pi$ admits a world view ${\cal A}$ such that $F$ is true in every answer set in $\A$. The answer sets can be computed by any variant of the answer set semantics while retaining all formal properties of the approach.

The work \cite{ShenE16} thoroughly studies the complexity of the proposed semantics, taking as a case study the FLP semantics of \cite{FaberPL11} for answer set programming, and establishes that deciding whether a program has a world view is at the third level of the polynomial hierarchy. They observe that complexity falls at the second level of the polynomial hierarchy if one refers to classical AS semantics, i.e., if one excludes explicit disjunction and additional programming constructs such as aggregates. 

For the semantics of \cite{ShenE16}, the work \cite{Woltran2018} proposes an ingenious method to compute and enumerate all the world views via an ASP solver. Precisely, they define a metaprogram which has first a stage of guessing a truth assignment of epistemic literals, and then a phase of checking. The method is effective, needs just one call to an ASP solver, and it is also optimal from a complexity-theoretic point of view. As this encoding makes use of large non-ground rules, specialized grounding methods must be adopted.

\section{Observations}
\label{observations}

In this section we propose some observations concerning programs with epistemic disjunction under the semantics of \cite{ShenE16}, applied however for the sake of simplicity to plain AS semantics. Thus, the kind of programs that we consider are ASP programs where: 
(i) there is no explicit disjunction in the head or body of rules;
(ii) both default negation (that we indicate with $\no$) and epistemic negation \N may appear (only) in the body of rules; 
(iii) concerning nested negations, we accept $\no \N A$ ($\K A$), $\N \no A$ ($\M A$), and $\no \N \no A$ ($\K\, \no A$, i.e., $\no \M A$), where $A$ is an atom; we do not consider $\no \no A$ for lack of space;
(iv) like \cite{ShenE16}, we do not explicitly consider \emph{classical negation} $\neg A$.

Consider for instance the following programs, aimed to study how world views are formed in similar though different significant cases.

\smallskip
$
\begin{array}{l}
\Pi_1\\
a\,\ar\ \no b.\tbl b\,\ar\ \no d.\tbl d\,\ar\ \no b.
\end{array}
$

\smallskip
$
\begin{array}{l}
\Pi_2\\
a\,\ar\ \N b. \tbl
b\,\ar\ \no d.\tbl
d\,\ar\ \no b.
\end{array}
$

\smallskip

$
\begin{array}{l}
\Pi_3\\
a\,\ar\ \N b, \no b.\ \ \,\,
b\,\ar\ \no d.\tbl
d\,\ar\ \no b.
\end{array}
$

\smallskip
$\Pi_1$ (involving no epistemic literals) has the unique world view $\{\{b\},\{a,d\}\}$
 coinciding with its set of answer sets. It can be easily verified that $\Pi_2$ (involving the epistemic literal $\N b$) has the unique world view $\{\{a,b\},\{a,d\}\}$ as in fact there exists an answer set of the even cycle involving $b$ and $c$ where $b$ is false: this makes \N $b$ true, and therefore $a$ can always be derived. Instead, $\Pi_3$ has the unique world view $\{\{b\},\{a,d\}\}$, coinciding with the one of $\Pi_1$. This because the presence of $\no b$ in the body of the first rule of $\Pi_3$ makes literal \N $b$ completely irrelevant: in fact, if we suppose that \N $b$ is true then the truth value of the head is determined by $\no b$; if \N $b$ is false then so is $\no b$, as there can be no answer sets where $b$ is false. Therefore, for the sake of simplicity and without loss of generality we assume that the body of each rule may include either \N $A$ or $\no A$, where $A$ is an atom, but not both.  
We assume also that $\K$ and $\M$ are applied to atoms occurring as the head of some rule in given program. Otherwise in fact, $\K A$ and $\M A$ are trivially false while their negations are trivially true.

It is important to notice that a given program may admit multiple world views only in presence of conflicting epistemic assumptions, like, e.g., in the example below:

$
\begin{array}{l}
a\,\ar\ \N b.\tbl
b\,\ar\ \N a.
\end{array}
$

\noindent
where in fact there are the two world views $\{\{a\}\}$ and $\{\{b\}\}$. In the variant 

$
\begin{array}{l}
a\,\ar\ c.\tbl
c \ar \N b.\tbl
b\,\ar\ d.\tbl
d \ar \N a.
\end{array}
$

\noindent
there are the two world views $\{\{a,c\}\}$ and $\{\{b,d\}\}$ where however the conflicting assumptions are still $\N a$ and $\N b$, as the positive dependencies are to all effects irrelevant.

\smallskip
 In order to assess which are the guesses that can be sensibly made to find world views, we propose the following method: (a) simplifying the given program so as to put into evidence cycles on epistemic literals; (b) finding the answer sets of such simplified version, and (c) extracting potentially valid guesses from these answer sets. As we will show, if valid guesses exist they correspond to some of these answer sets, whose number thus establishes an upper bound on the number of valid guesses and consequently of world views. 

\begin{definition}[Epistemic simplified version of a program]
	\label{episimple}
	Given program $\Pi$ including epistemic literals\footnote{If given program does not contain epistemic literals then it admits a unique world view.}, the \emph{epistemic simplified version of\ \,$\Pi$}, that we call $\Pi^s$,
	is obtained via the following steps.
	\begin{enumerate}
		\item Preliminary step, for \M\ and\ \K\ elimination: replace each literal $\N\no A$ with $\N A'$, where $A'$ is a fresh atom; cancel each literal $\no \N A$.
		\item Simplification, composed of the following steps:
		\begin{enumerate}
			\item[(i)] Fold all positive dependencies internal to cycles, following the criteria provided in \cite{CostantiniP05}; folding can be performed efficiently, for instance by exploiting the techniques discussed in \cite{BrassDFZ01} \footnote{Folding techniques may lead to exponential growth of program size but this does not happen if they are applied only inside cycles (cf. \cite{CostantiniP05}). We do not provide details about how such techniques are applied, because the step of folding is not really necessary for the proposed method to work: it is performed here only for the sake of clarity, in order to get rid of irrelevant non-epistemic literals and thus make Definition\rif{episcen} (seen below) simpler.}.
			\item[(ii)] Eliminate all rules with no epistemic literals in the body.
			\item[(iii)] For each remaining rule $\rho$, substitute the conjunction of non-epistemic literals in the body of the rule (if present) by a fresh atom $a_{\rho}$; for each such fresh atom, add the fresh even cycle $\{a_{\rho} \ar \no noa_{\rho}.\ \ noa_{\rho} \ar \no a_{\rho}.\}$ where also $noa_{\rho}$ is a fresh atom, to signify that the non-epistemic body of each rule can be either true or false.
			\item[(iv)] For each epistemic literal $\N C$ such that $C$ does not occur in the head of any of the rules of $\Pi^s$ as obtained in previous steps (i.e., $\N C$ is not involved in cyclic dependencies with other epistemic literals; notice that this case includes the $\N A'$ since $A'$ cannot occur in rule heads) substitute $\N C$ with fresh atom $NC$ and add the fresh even cycle $\{C \ar \no NC.\ \ NC \ar \no C.\}$ to signify that $\N C$ can be assumed or not.
			\item[(v)] Substitute \N with $\no$.
		\end{enumerate}
	\end{enumerate}
\end{definition}

After computing $\Pi^s$ (point (a) of the proposed method), we compute (point (b)) the answer sets $M_1, \ldots ,M_n$ of $\Pi^s$, and then (point (c)) we determine the potentially valid guesses related to these answer sets. Since in building $\Pi^s$ the epistemic literals occurring in the original program $\Pi$ have been transformed so as to admit only $\no$ as negation, in this step we restore such literals to their original form, and we get rid of the fresh atoms introduced in $\Pi^s$.

\begin{definition}[Epistemic scenarios]
	\label{episcen}
	Given program $\Pi$ and its epistemic simplified version $\Pi^s$, for each answer sets $M_i$ of $\Pi^s$, $i\leq n$,	
	
$
\begin{array}{l}
S_i = \ \ \{\N A |\, A \not\in M_i\}\ \cup
\ \{\N A |\, NA \in M_i \}\ \cup
\ \{\N \no A |\, NA' \in M_i\}\ \cup\\
\tbm\ \ \{\no\N A |\, A \in M_i \mbox{\ and $\no\N A$ occurs in $\Pi$}\}
\ \cup\\
\tbm\ \ \{\no\N \no A |\, A' \in M_i \mbox{\ and $\no\N \no A$ occurs in $\Pi$}\}
\end{array}
$

We call the $S_i$s the \emph{epistemic scenarios} for $\Pi$, and $\Sc$ the set of such sets. We call $\Sc_M$ the set of maximal scenarios, i.e., \(\Sc_M = \{S_i \in \Sc : \nexists S_j\in \Sc, S_i\subset S_j\}\).
\end{definition}

We can prove that: 

\vspace{-0.1cm}
\begin{theorem}
	\label{guesses}
	For every world view $\A$ of a given program $\Pi$ with set of epistemic scenarios $\Sc$, if $\A$ is obtained from guess $\Phi$ then $\Phi \in \Sc$.
\end{theorem}

\vspace{-0.2cm}
\noindent\
\begin{proof}
(sketch) As it is well known, a program may have several answer sets only if the program includes cycles on negation. In fact, in a cycle such as $a \ar \no b.\ b\ar \no a$, only one of $\no a$ or $\no b$ can be deemed to be true in each answer set (in absence of other rules supporting $a$ or $b$ or both). The semantics of \cite{ShenE16} computes world views by means of the epistemic reduct, which is very similar to the (Gelfond-Lifschitz) reduct for computing answer sets, and thus presents the same idiosyncrasies. In fact, consider, e.g., a cycle $a \ar \N b, a_{\rho}\ \ b\ar \N a,b_{\rho}$\  where the fresh atoms $a_{\rho}$ and $b_{\rho}$ are assumed to be true. If making a guess including both $\N a$ and $\N b$ (i.e., assuming both of them to be true) then both $a$ and $b$ would belong to any answer set, thus determining a contradiction since both epistemic literals should then be false. So in this case, each valid guess may involve either $\N A$ or $\N B$, but not both. Thus, guesses which can possibly be valid (and thus generate a world view) must include sets of epistemic literals not conflicting with each other in cycles. Epistemic literals not involved in cycles are independent of the others, and may thus be freely assumed as true/false in any guess. Assumptions on nested negations follow in a straightforward way from assumptions on \N\hspace{-0.1cm}. In fact, the proposed method determines the set $\Sc$ of potentially valid guesses on a simplified version of the program, obtained as specified in Definition\rif{episimple}. This simplified version is focused on epistemic literals though faithful to the program intended meaning, and thus makes it possible to compute via plain answer set computation all those guesses potentially corresponding to world views.
\end{proof}

\begin{corollary}
	\label{numguesses}
	Let {\bf n} be the number of atoms $A$ occurring in epistemic negations $\N A$ and $n$ the number of rule heads in the 
	given program $\Pi$, and let $\hat{n}$ be the number of rule heads occurring in its epistemic simplified version $\Pi^s$. The maximum number of valid guesses leading to world views of $\Pi$ is $\Theta(3^{\hat{n}/3})$, where we have $\Theta(3^{\hat{n}/3}) \simeq \Theta(3^{\mbox{\bf n}/3}) \simeq \Theta(3^{n/3})$.
\end{corollary}

This follows from the famous
result by \cite{CholewinskiT99} on the maximum number of answer sets of a given program that, they show, depends on cycles over negation. In fact, the valid guesses are among the epistemic scenarios $\Sc$, which are extracted from the answer sets of $\Pi^s$ whose number is $\Theta(3^{\hat{n}/3})$. The cycles over negation in $\Pi^s$ correspond (by Definition \rif{episimple}) to cycles over epistemic negations in $\Pi$, except for the addition of few fresh even cycles that do not significantly increase the size of $\Pi^s$ w.r.t. the size of $\Pi$. 

\smallskip
\noindent{\bf Our Contribution:} Corollary\rif{numguesses} \emph{provides for the first time in the literature an upper bound on the number of valid guesses}.

\smallskip
Not all epistemic scenarios will in general correspond to valid guesses, as validity of a guess will depend upon effective true/falsity of non-epistemic literals occurring in rules. Validity can be verified as stated in \cite{ShenE16} and mentioned in Section\rif{elp}, or via the new method introduced below in Section\rif{elpras}. Theorem\rif{guesses} however allows one to reduce the number of guesses to be considered. 
Though this cannot reduce complexity, it can nonetheless be of practical usefulness. Finding epistemic scenarios may provide in the average case (when not all guesses are valid) an advantage if applied, e.g., as a pre-processing stage for the method proposed in \cite{Woltran2018}. In fact, this method might thus consider, represent and check fewer guesses, where each guess is represented by a long rule in the metaprogram that they adopt, and is therefore computationally heavy. Since however computing epistemic scenarios has a cost, the effective usefulness will have to be evaluated by experiments that assess \as how small'' {\bf n} should be in order to obtain an advantage. Or, one might allot a predefined threshold proportional to {\bf n} to the running time of the pre-processing stage, that would be stopped if exceeding this time: so, pre-processing would bring either a substantial advantage or a negligible disadvantage.

\section{Resource-based Answer Set semantics (RAS)}
\label{ras}

Resource-based Answer Set semantics (RAS, presented in \cite{CostantiniF15}) has the property to provide answer sets to every program (so, under RAS there are no inconsistent programs).  The principles underlying the RAS semantics are the following: (i) atoms belonging to a RAS answer set are either definitely true (i.e., a RAS answer set includes all atoms that are true w.r.t.\ 
the well-founded semantics\footnote{The well-founded semantics \cite{VGelderRS91} provides to every program a unique three-valued model $\langle  W^{+}, W^{-}\rangle$,
	where atoms in $W^{+}$ are \emph{true}, those in $W^{-}$ are \emph{false},
	and all the others are \emph{undefined}. All atoms in $W^{+}$ are true in every AS answer set, and all atoms in $W^{-}$ are false; the answer set semantics (AS) in fact assigns,
	for consistent programs,
	truth values
	to the undefined atoms, and so does RAS, though in a slightly different way.}), or have been rationally (though defeasibly) assumed to hold;
(ii) atoms not included in a RAS answer set are either definitely false (i.e., false w.r.t.\ 
the well-founded semantics), or have been rationally assumed not to hold in order to draw some conclusion, or no judgment about them has been devised because any such assessment would lead to contradiction. For instance, program $\{p \ar \no p.\}$ has an empty RAS answer set because, in our view, a rational agent cannot believe both $\no p$ and $p$, so no rational judgment on truth value of $p$ can be given; differently from AS, this is not seen as a reason not to provide a semantics. Therefore, a RAS answer set: includes atoms either proved or assumed to be true; does not contain atoms either proved or assumed to be false, or whose truth value cannot be assessed. 

To formally explain the difference between AS and RAS we may resort to a modal logic formulation. For AS, as discussed in
\cite{MarekT93} an answer set programming rule
can be transposed, to express its logical meaning, into its \as modal image'', where $L\,A$ is intended as \as $A$ is believed'' under any modal logic contained in {\bf S5}:

\smallskip
\(L\,A_1 \wedge \cdots \wedge L\,A_n \wedge L\,\neg\,L\,B_1 \wedge \cdots \wedge L\,\neg L\,B_m \supset L\,A\ \,\,\, ~~(Ae0)\)

In RAS, as discussed in \cite{CostantiniF15}, the logical meaning of each rule $\rho$ is expressed by
the following couple of modal rules, that form its modal image:

\smallskip
\(
\begin{array}{ll}
L\,A_1 \wedge \cdots \wedge L\,A_n \wedge L\,\neg\,L\,B_1 \wedge \cdots \wedge L\,\neg L\,B_m \supset L\,\dot A\tbs&(Ae1)\\
L\,\dot A \wedge \neg\,L\neg\,L\,A \supset L\,A\tbs&(Ae2)
\end{array}
\)

$(Ae1)$ modifies $(Ae0)$ in the sense that, based on the same premises, one concludes $L\dot A$, which means that one believes \emph{to be enabled} to prove $A$.
$(Ae2)$ states that $L\,A$ is derived only if $L\,\dot A$ holds, and one does not believe not to believe $A$. Thus, in the case for instance of the unary odd cycle $p \ar \no p$ which makes an ASP program inconsistent, by $(Ae1)$ one from $\no p$, that in terms of the modal image is expressed as $L\,\neg\,L\,p$, can derive $L\,\dot p$, i.e., to be enabled to prove $p$. However, one cannot do so, as $L\,\neg\,L\,p$ precisely accounts to believing not to believe $p$. Consequently $(Ae2)$ cannot be applied, and $p$ therefore is false without raising inconsistencies. 

Each AS answer set is also a RAS answer set, but a RAS answer set is not necessarily a classical model of the program. 
RAS answer sets are in fact all the \as Maximal Consistently Supported'' (MCS) sets of atoms that a given program $\Pi$ admits. I.e., for each RAS answer set $M$ of $\Pi$: every atom $A$ in $M$ is \emph{consistently supported} by a set $S$ of rules of $\Pi$, where each rule in $S$ is supported in $M$, $A$ is the head of exactly one rule in $S$, and each atom (positive literal) in the body of rules in $S$ is different from $A$, and is in turn consistently supported; moreover, $M$ is maximal, meaning that there exists no $M'$ with the same properties such that $M \subset M'$. E.g., for program \(\{a \ar \no b.\ \ b\ar \no c.\ \ c\ar \no a.\}\), the sets of atoms $\{a\}$, $\{b\}$ and $\{c\}$ are the only MCSs and thus they are the RAS answer sets of this program. In fact, each of the composing atoms $a,b,c$'s is consistently supported by set $S$ consisting of the single rule of which that atom is the head and the three sets are maximal, as all their supersets (among which are the classical models) are not consistently supported. We made the examples of unary and ternary odd cycles because what makes programs inconsistent under AS are indeed (direct or indirect) odd loops, i.e., when an atom depends upon its own negation through an odd number of negative dependencies.

RAS answer sets of program $\Pi$ can be computed via an operator $\hat{\Gamma}_{\Pi}$ which is analogous to ${\Gamma}_{\Pi}$ though it takes as input those interpretations $I$ such that for every $A\in I$ there exists rule $\rho$ in $\Pi$ with head $A$, and is based upon: a modified reduct (that, with respect to the `traditional' reduct recalled in Section\rif{asp}, does not perform step 2), and a modified $T_{\Pi}$ operator, that computes consistently supported sets of atoms by discarding those atoms for which all possible derivations depend on their own negation. In fact, as discussed before such atoms are excluded from any RAS answer set because their truth value cannot be assessed.
$M$ $=$ $\hat{\Gamma}_{\Pi}(I)$ is a RAS answer set
iff $M \subseteq I$ and $M$ is maximal, i.e., there is no proper subset $I_1$ of $I$ that determines $M' \supset M$.

On programs which are consistent under AS, `traditional' answer sets are among RAS answer sets, that can be more numerous. E.g., program $\{a\,\ar\ \no b.\ b\,\ar \no a.\ p\,\ar\ \no p,\,a.\}$ has unique answer set $\{b\}$ under AS and answer sets $\{a\}$ and $\{b\}$ under RAS. For AS, falsity of $a$ is required in order to bypass the odd cycle by falsifying $p$. For RAS this is no longer required, as $p$ is excluded from any answer set because it depends on its own negation. Odd cycles are in fact the only possible source of difference between RAS and AS on programs which are consistent under AS. For programs that are either 'call-consistent' (i.e., they do not involve odd cycles) or that fulfill straightforward syntactic sufficient conditions concerning odd cycles, the answer sets returned by the two semantics are exactly the same. So, RAS can be seen as a variant of AS. Sufficient conditions are based on the fact that, as discussed in \cite{Cos06}, a program including odd cycles can be consistent under AS only if for each odd cycle one of the following conditions hold: (i) some rule involved in the cycle has conjuncts that, if false, force falsity of the rule head or (ii) some atom involved in the cycle is the head of a rule outside the cycle where conjuncts in its body, if true, force truth of the head atom; in both cases, the cycle is 'broken' and becomes harmless, where these conjuncts are called \as handles'' of the cycle. Thus, a first sufficient condition requires that atoms appearing in the handles of some odd cycle do not occur, either directly or indirectly through dependencies, in other cycles or in their handles: this in fact ensures that different cycles do not interact. A second sufficient condition, which is more restrictive but easier to verify, requires that atoms occurring in handles of odd cycles do not occur elsewhere in the program. In perspective, as future work it is possible to study and possibly devise other, more fine-grained, sufficient conditions.

Differently from AS, RAS enjoys the properties that, ideally, every non-monotonic formalism should enjoy \cite{Dix95AeB}: (i) cumulativity, i.e., the possibility of asserting lemmas while keeping the same answer sets; (ii) relevance, i.e., the fact that the truth value of each atom is determined by the subprogram consisting of the \emph{relevant rules}, which are those upon which the atom depends (directly or indirectly, positively or negatively). AS does not enjoy such properties because some atoms may be forced to assume certain truth values in order to prevent inconsistencies; this is no longer the case for RAS.
An advantage of the property of relevance is to make top-down query-answering possible \cite{CostantiniF16} in prolog-style, i.e., without computing the answer sets in advance but rather via an enhanced resolution procedure. This because in RAS only the relevant rules have a role in proving/disproving any atom $A$. Instead, under AS a query $? A$
might \emph{locally} succeed, but still, for the lack of relevance, the overall program may not have answer sets including $A$. Under RAS, a query $? A$ w.r.t. ground program $\Pi$ asks whether $A$ is true in (belongs to) some answer set of $\Pi$. Query $\no A$ asks whether $\no A$ is true in some answer set of $\Pi$, which implies that there exists some answer set to which $A$ \emph{does not} belong. Series of queries can be, upon user's choice: (i) \emph{contextual}, i.e., query $? A, B$ asks whether $A$ is true (belongs to) some answer set, and $B$ is true in (belongs to) some of those; (ii) independent. 

RAS query-answering is performed \cite{CostantiniF16} via RAS-XSB-resolution. It is implemented on top of XSB-resolution \cite{SwiftW12,ChenW93}, which is
an efficient, fully described and implemented procedure, correct and
complete w.r.t.\ the well-founded semantics.
Features of XSB-resolution that are crucial for the implementation of RAS-XSB-resolution are negative cycles detection and the tabling mechanism, that associates to program $\Pi$ a table \tabp, which is initialized prior to posing queries. Such table
contains information about true and false atoms useful for both the present and the subsequent queries (if they are in conjunction). The principle of functioning for table \tabp\ under RAS-XSB-resolution are the following. During a proof:
(a) the negation of any atom
	which is not a program fact is available unless this atom has been proved;
(b) the negation of an atom which has been proved becomes unavailable, and the atom is asserted as true; so, (c) the negation of an atom which cannot be proved remains always available. Since RAS-XSB-resolution is a top-down proof procedure, modifications to the table might be undone and redone differently upon backtracking. The present (prototype) implementation exploits XSB (or, more precisely, its basic version XOLDTNF), as a \as plugin'' for definite
success and failure, where new cases are added to manage atoms with truth value \emph{undefined} under XSB. RAS-XSB-resolution is correct and complete w.r.t.\ resource-based answer set semantics,
in the sense that, given a program $\Pi$, a query $? A$ succeeds under RAS-XSB-resolution
with an initialized \tabp\  iff there exists some resource-based answer set $M$ for $\Pi$
where $A\in M$. The result extends to sequences of queries.

A consequence of rendering odd cycles consistent is that under RAS constraints must be defined as such, since they can no longer be rephrased in terms of unary odd cycles. So, constraints must be associated to the program as an additional layer. However, as RAS answer sets computation is a variant of answer set computation, constraints can still be checked during computation as solvers usually do. In what follows, when clear from the context, we will call RAS answer sets simply `answer sets'. 

\section{Epistemic Negation in RAS}
\label{elpras}

It can be easily seen that the approach of \cite{ShenE16} is applicable to RAS (in fact, the approach has been devised in order to be applicable to \emph{any} variant of the AS semantics). Indeed, as seen in Section\rif{elp}, after performing the epistemic reduct $\Pi^{\Phi}$ of given program $\Pi$ w.r.t. a guess ${\Phi}$, the set ${\cal A}$ of answer sets of $\Pi^{\Phi}$ is computed, which is a candidate world view to be checked for validity and maximality. If, instead of computing the answer sets, one computes RAS answer sets, or even if one adopts some other semantics which produces sets of sets as result, all the rest remains unchanged. In fact, in Section\rif{observations} we have shown that epistemic scenarios can be found by considering epistemic literals only.

In this Section we provide the following {\bf contributions}:
\noindent
\begin{enumerate}
	\item
	We define a method to check validity of guesses without computing the world views, via RAS query-answering.
	\item
	We introduce more generally the possibility to {\bf query world views, and to query the whole set of world views}, and we define useful operators.
\end{enumerate}

So, a reader can either suppose to adopt RAS instead of AS, or (s)he can assume that for any given program $\Pi$ the AS and RAS answer sets coincide (for instance, because $\Pi$ fulfills one of the above-mentioned sufficient conditions).

Consider first a program $\Pi$ which does not contain epistemic negation. Its unique world view will thus coincide with the set of its answer sets. Via RAS-XSB-resolution, we are able to implement epistemic queries, among which the following (where $A$ is an atom):

\begin{itemize}
	\item 
	Query  ?\,\N $A$ asks whether $A$ is false w.r.t. some answer set of $\Pi$, and therefore succeeds if $\no A$ is true in some of them. This can be implemented via RAS query $? \no A$.
	
	\item 
	Query ?\,$\N \no A$ asks whether $\no A$ is false in some answer set, and therefore succeeds if $A$ is true in some of them, which corresponds to query ?\,$\M A$. This can be implemented via RAS query $? A$.
	 
	\item 
	Query ?\,$\no$ \N $A$ asks whether it is not true that $A$ is false w.r.t. some answer set of $\Pi$, i.e., that $A$ is true in all of them, which corresponds to ?\,$\K A$. This can be implemented via RAS query $? A, \no A (fail)$, where this overall query succeeds if ?\,$A$ succeeds, i.e., $A$ belongs to some answer set of $\Pi$, whereas ?\,$\no A$ fails, so $A$ is not false in any of them. 
	
	\item Query ?\,$\no \N \no A$ asks whether $A$ is false in every answer set, meaning $\K\, \no A$, i.e., $\no \M A$. This can be implemented via RAS query $? A (fail)$ that succeeds if $A$ fails, i.e., exactly whenever there is no answer set where $A$ is true. We introduce a new operator \NO as a shorthand for $\no \N \no A$.
\end{itemize}

We will now proceed to consider programs including epistemic negation. 
We have seen in Section\rif{observations} how to identify plausible guesses for given program $\Pi$. Then one has to verify, for each guess $\Phi$, that it indeed corresponds to a world view. To check a guess under RAS, one can exploit the following method, alternative to computing all the answer sets of $\Pi^{\Phi}$.

\begin{theorem}[RASCGK test: RAS Candidate Guess check]
	Potential validity of a guess $\Phi$ w.r.t. program $\Pi$ (i.e., the fact that $\Phi$ corresponds to a candidate world view) can be checked as follows.
	\label{rasguesscheck}
\begin{enumerate}
	\item
Derive $\Pi^{\Phi}$ as said in Section\rif{elp}, i.e., (i) delete from $\Pi$ all the epistemic literals belonging to $\Phi$ and, (ii) for all the epistemic literals which occur in $\Pi$ but are not in $\Phi$, substitute each of them with a fresh atom (a shortcoming for what done in \cite{ShenE16}, where in such literals they substitute $\N$\hspace{-0.1cm} with $\no$\hspace{-0.1cm}, which accounts to considering them false).
\item
Given $\Pi^{\Phi}$: for every epistemic literal $\N A$  occurring in $\Pi$ pose the query $? not A$; for every epistemic literal $\N \no A$  occurring in $\Pi$, pose the query $? A$;
for every epistemic literal $\no \N \no A$  occurring in $\Pi$, pose the query $? A (fail)$; and, for every epistemic literal $\no \N A$ pose the query $? A, \no A (fail)$. Guess $\Phi$ is a \emph{candidate valid guess} if all queries concerning epistemic literals ${\cal L} \in \Phi$ succeed, while instead all the others fail.
\end{enumerate}
\end{theorem}
\begin{proof}
Straightforward given the above observations about the meaning of queries.
\end{proof}

In order to establish validity of a guess ${\Phi} \in \Sc$, one has to establish that ${\Phi}$ passes the RASCGK test whereas no superset $\Phi' \supset \Phi$ with ${\Phi'}\in \Sc$ does. In order to find \emph{all} valid guesses, one can do the following: (i) check epistemic scenarios in $\Sc_M$ (which are the largest ones); those which pass the RASCGK test are valid guesses; (ii) for every $\bar{\Phi} \in \Sc_M$ which does not pass the test, check via the RASCGK test every $\bar{\Phi'} \in \Sc$ such that $\bar{\Phi'} \subset \bar{\Phi}$ starting from the sets with greater cardinality. Notice that the empty set may be a valid guess.

To implement world views querying, we exploit the tabling mechanism of XSB- and RAS-XSB-resolution, that as mentioned associates to program $\Pi$ a table \tabp, initialized prior to program execution. In order to query the program under a certain valid guess $\Phi$, the initialization is customized accordingly, by setting to true all and only the epistemic literals occurring in the valid guess under consideration.

\begin{definition}[Guess-tailored table initialization]
Given program $\Pi$ and a valid guess $\Phi$, \emph{guess-tailored table initialization} will be performed in addition to normal initialization, in the following way: each epistemic literal occurring in $\Phi$ will be set to true, and all the other epistemic literals will be set to false. 
\end{definition}

\begin{definition}[G-RAS-XSB-Resolution]
The variant of RAS-XSB-resolution where, given program $\Pi$ and valid guess $\Phi$, guess-tailored table initialization is applied is called G-RAS-XSB-resolution (tailored to $\Phi$).
\end{definition}

\begin{theorem}
Given program $\Pi$ and a valid guess $\Phi$, G-RAS-XSB-resolution tailored to $\Phi$ is correct and complete w.r.t. the world view $\A$ obtained from $\Phi$.
\end{theorem}

\begin{proof}
(sketch) RAS-XSB-resolution is correct and complete w.r.t. the answer sets of given program $\Pi$. Thus, the RASCGK test is able to correctly establish whether a guess $\Phi$ is valid. Since a world view consists of the (set of) answer sets of given program where however epistemic literals belonging to valid guess $\Phi$ are deemed to be true and the other epistemic literals are deemed to be false, and since guess-tailored table initialization does exactly so, then G-RAS-XSB-resolution is correct w.r.t. the world view resulting from $\Pi$ given $\Phi$.
\end{proof}

For instance, to check whether $F$ is true in $\Pi$ under the general epistemic semantics of \cite{ShenE16}, i.e., to check whether $F$ belongs to every answer set of a the world view, it will suffice to issue the query ?\,$\K F$ to this world view.

Moreover, via a further extension one can be able to query the whole set of world views. 

\begin{definition} [Extended multi-view program]
	Let $\Pi$ be a program, and $\Phi_1,\ldots,\Phi_k$ be the valid guesses for $\Pi$. Let the \emph{extended multi-view program} ${\Pi}^{E}$ be an extended program obtained as the union of $k$ copies $\Pi^1,\ldots,\Pi^k$ of $\Pi$, where such copies have been however previously \emph{standardized apart}, i.e., atoms occurring therein have been suitably renamed. Let as assume for instance that each atom $A$ in $\Pi$ becomes, in $\Pi^1,\ldots,\Pi^k$, respectively $A^1,\ldots,A^k$. 
\end{definition}

\begin{definition}[Multi-guess table initialization]
	Given program $\Pi$, given its associated valid guesses $\Phi_1,\ldots,\Phi_k$, and given the \emph{extended multi-view program} ${\Pi}^{E}$, \emph{multi-guess table initialization} will be performed by: performing guess-tailored table initialization w.r.t. each $\Phi_1,\ldots,\Phi_k,$ where however these guesses have been previously \emph{standardized apart} correspondingly to what done for ${\Pi}^{E}$; i.e., each atom $A$ occurring in some literal of $\Phi_i$ ($i \leq k$) becomes $A^i$. 
\end{definition}

\begin{definition}[W-G-RAS-XSB-Resolution]
	The variant of G-RAS-XSB-resolution where, given program $\Pi$ and valid guesses $\Phi_1,\ldots,\Phi_k$, multi-guess table initialization is applied to ${\Pi}^{E}$, is called W-G-RAS-XSB-resolution (tailored to $\Phi$).
\end{definition}

Now, by W-G-RAS-XSB-resolution it is possible to issue the following queries on given program $\Pi$:

\begin{itemize}
\item Query ?\,$\Kw A$ asks whether $A$ is true in all world views of $\Pi$. This query is translated into query ?\,$\K A^1, \ldots, \K A^k$ to be executed on ${\Pi}^{E}$ (where comma stands for conjunction, i.e., for this query to succeed all conjuncts must succeed).
\item Query ?\,$\Mw A$ asks whether $A$ is possible in some world view of $\Pi$. This query is translated into query ?\,$\M A^1; \ldots; \M A^k$ to be executed on ${\Pi}^{E}$ (where semicolon stands for disjunction, i.e., for this query to succeed at least one disjunct must succeed). We might of course modify \Mw so as to check whether $A$ is possible in \emph{every} world view, by substituting `;' with `,'.
\item 
Query  ?\,\Nw $A$ asks whether $A$ is false w.r.t. some world view of $\Pi$, and therefore if $\no A$ is true in some of them. So, this corresponds to query $? \no A^1;\ldots;\no A^k$ to be executed on ${\Pi}^{E}$ (where semicolon again stands for disjunction).
\item 
Query ?\,\NOw $A$ asks whether $A$ is false in every answer set of every world view, translated under W-G-RAS-XSB-resolution as $?\, \NO A^1,\ldots,\NO A^k$.
\end{itemize}

This amounts to introducing new epistemic operators, that can be useful in practice as shown by the example presented in the next section.

\section{Case Study}
\label{casestudy}

In this section we will reconsider the example presented in the Introduction. Let us elaborate on this example, stating that a person is guilty if a witness recognizes that person (say, when perpetrating the crime):

$
\begin{array}{l}
\mathit{guilty(X)} \ar \mathit{suspect(X)}, \mathit{witness\_recognizes(X)}.\\
\mathit{witness\_recognizes(X)} \ar \mathit{suspect(X)}, \mathit{witness1\_recognizes(X)}.\\
\mathit{witness\_recognizes(X)} \ar \mathit{suspect(X)}, \mathit{witness2\_recognizes(X)}.
\end{array}
$

\smallskip
Assume to have the fact $\mathit{suspect(john)}$.

Now, if we have also a fact, e.g., $\mathit{witness1\_recognizes(john)}$, then this will be true in every answer set and thus $\mathit{guilty(john)}$ will be provable. Assume instead that witness number one is not sure, which can be represented as:

$
\begin{array}{l}
\mathit{witness1\_recognizes(john)} \ar \no \mathit{witness1\_not\_recognizes(john)} \\
\mathit{witness1\_not\_recognizes(john)} \ar \no \mathit{witness1\_recognizes(john)}.
\end{array}
$

\smallskip
In this case, if we have no information about evidence provided by the other witness, then $\mathit{innocent(john)}$ will be provable, as there is an answer set where both $\mathit{witness\_recognizes(john)}$ and thus $\mathit{guilty(john)}$ are false. 

Assume now that we have $\mathit{witness1\_recognizes(john)}$.
Assume that instead we have been informed that $\mathit{witness2\_recognizes(john)}$ is false. So, we can state that the two disagree via a fact $\mathit{disagree\_w1\_w2(john)}$. 
If we do not want to give priority to any of the witnesses, in absence of other information we may state:

$
\begin{array}{l}
\mathit{guilty(X)} \ar \mathit{suspect(X)}, \mathit{witness\_recognizes(X)}.\\
\mathit{witness\_recognizes(X)} \ar \mathit{witness1\_recognizes(X)},\mathit{reliable(witness1,X)}.\\
\mathit{witness\_recognizes(X)} \ar \mathit{witness2\_recognizes(X)},\mathit{reliable(witness2,X)}.
\end{array}
$

\smallskip\noindent
Whenever two witnesses propose contrary evidence on someone/something, then they cannot be both reliable (concerning the situation at hand). This can be expressed as follows:

$
\begin{array}{l}
\mathit{reliable(witness1,X)} \ar \mathit{disagree\_w1\_w2(X)},\NO \mathit{reliable(witness2,X)}.\\
\mathit{reliable(witness2,X)} \ar \mathit{disagree\_w1\_w2(X)}, \NO \mathit{reliable(witness1,X)}.
\end{array}
$

\smallskip
We use the new proposed operator \NO because \N is too weak: it is not sufficient, to deem a witness to be reliable, to state that the other one \emph{may be} unreliable. In such case as this one should seek certainty, i.e., \NO checks falsity in \emph{every} answer set of a world view. Notice that the above formulation results in a cycle on the operator \NO\hspace{-0.1cm}, namely, in abstract terms
\(
\{a \ar \NO b.\ b \ar \NO a.\}\)
which gives rise to alternative world views for exactly the same reasons stated in Theorem\rif{guesses} for \N\hspace{-0.1cm}.
So, at the present stage we have two world views:

\begin{description}
\item $
\begin{array}{l}
\{\{\mathit{suspect(john), witness1\_recognizes(john)},\mathit{disagree\_w1\_w2(john)},\\
\ \ \ \, \mathit{reliable(witness1,john)},\mathit{witness\_recognizes(john)},\mathit{guilty(john)}\}\}
\end{array}$
\item and
\item $
\begin{array}{l}
\{\{\mathit{suspect(john),witness1\_recognizes(john)},\mathit{disagree\_w1\_w2(john)},\\
\ \ \ \,\mathit{reliable(witness2,john)},\mathit{innocent(john)}\}\}
\end{array}$
\end{description}

We can see that in the first world view $\mathit{john}$ is guilty, while in the second one he is not. Here we can appreciate the usefulness of the queries over the whole set of world views:

\begin{description}
	\item ?\,$\Kw \mathit{guilty(john)}$ asks whether $\mathit{john}$ is guilty in every world view, and the answer is negative;
	\item ?\,$\Mw \mathit{guilty(john)}$ asks whether $\mathit{john}$ is possibly guilty in some world view, and the answer is positive.
\end{description}

We can rephrase the original example as follows:

$
\begin{array}{l}
\mathit{innocent(X)} \ar \mathit{suspect(X)}, \mathit{provably\_innocent(X)}.\\
\mathit{innocent(X)} \ar \mathit{suspect(X)}, \mathit{presumed\_innocent(X)}.\\
\mathit{provably\_innocent(X)} \ar \NOw \mathit{guilty(X)}.\\
\mathit{presumed\_innocent(X)} \ar \Nw \mathit{guilty(X)}.
\end{array}
$

\smallskip
In this formulation, $\mathit{john}$ is $\mathit{presumed\_innocent}$ because there exists indeed a world view where he is not guilty. However, he is not provably innocent because there will exist world views where he is guilty. A court would deem $\mathit{john}$ innocent in both cases. However, prior to court, the law enforcement might decide for supplementary investigations (for instance, concerning reliability of witnesses) in order to ascertain a conclusion `beyond any reasonable doubt'.

\section{Conclusions}

We have proposed alternative characterizations of ELPs (Epistemic Logic Programs), and we have shown the usefulness of our approach in formalizing significant examples. 
Future directions include the extension of the approach so as to make the syntax of the programs that we consider more general. An implementation is also in order, to replace the embryonic prototype on which we have run the proposed examples. Then, experiments will be performed on suitable benchmarks in order to assess the effective practical value of the approach.

\end{document}